\documentclass[letterpaper, 10 pt, conference]{ieeeconf}

\IEEEoverridecommandlockouts
\usepackage{cite}
\usepackage{amsmath,amssymb,amsfonts}
\usepackage{algorithmic}
\usepackage{graphicx}
\usepackage{textcomp}
\usepackage{setspace}

\usepackage{graphics}
\usepackage{epsfig}
\usepackage{times} 

 \usepackage[dvipsnames]{xcolor}
\usepackage{xurl}
\usepackage{hyperref}
\usepackage{bm}
\usepackage{breqn}
\usepackage{amsmath}
\usepackage{amssymb}
\usepackage{tabularx}

\usepackage{amsthm}

\newcommand{\reals}{\mathbb{R}}
\newcommand{\R}{\reals}
\newcommand{\Rnonneg}{\reals_{\geq 0}}
\newcommand{\Rplus}{\reals_{>0}}

\renewcommand{\S}{\mathbb{S}}

\newcommand{\pd}{\S_{++}}
\newcommand{\psd}{\S_{+}}

\renewcommand{\det}[1]{\left|#1\right|}
\newcommand{\tr}[1]{\operatorname{tr}\left(#1\right)}

\newcommand{\Dcal}{\mathcal{D}}

\newcommand{\Lcal}{\mathcal{L}}

\newcommand{\Ncal}{\mathcal{N}}

\newcommand{\Scal}{\mathcal{S}}

\newcommand{\Ucal}{\mathcal{U}}

\newcommand{\Xcal}{\mathcal{X}}

\newcommand{\eqn}[1]{\begin{align} #1 \end{align}}
\newcommand{\eqnN}[1]{\begin{align*} #1 \end{align*}}

\newcommand{\bmat}[1]{\begin{bmatrix}#1\end{bmatrix}}

\newcommand{\norm}[1]{\left\Vert #1 \right \Vert}
\newcommand{\abs}[1]{\left | #1 \right |}

\newcommand{\argmax}[1]{\underset{#1}{\text{argmax}}}

\newcommand{\trace}{\operatorname{tr}}

\theoremstyle{plain}
\newtheorem{theorem}{Theorem}
\newtheorem{corollary}[theorem]{Corollary}

\theoremstyle{definition}
\newtheorem{definition}{Definition}

\newtheorem{remark}{Remark}
\newtheorem{property}{Property}

\theoremstyle{remark}

\newtheorem{example}{Example}

\begin{document}
\title{Sensor-based Planning and Control for Robotic Systems: \\Introducing Clarity and Perceivability}
\author{Devansh R. Agrawal, \IEEEmembership{Student, IEEE}, Dimitra Panagou, \IEEEmembership{Senior Member, IEEE}
\thanks{This work was supported by the National Science Foundation (NSF) under Grant 1942907. Both authors are at the University of Michigan, Ann Arbor, MI 48109 USA (e-mail: \{devansh, dpanagou\}@umich.edu). }
}


\maketitle

\begin{abstract}
We introduce an information measure, termed \emph{clarity}, motivated by information entropy, and show that it has intuitive properties relevant to dynamic coverage control and informative path planning. Clarity defines the quality of the information we have about a variable of interest in an environment on a scale of $[0, 1]$, and has useful properties for control and planning such as: (I) clarity lower bounds the expected estimation error of \emph{any} estimator, and (II) given noisy measurements, clarity monotonically approaches a level $q_\infty < 1$. We establish a connection between coverage controllers and information theory via clarity, suggesting a coverage model that is physically consistent with how information is acquired. Next, we define the notion of \emph{perceivability} of an environment under a given robotic (or more generally, sensing and control) system, i.e., whether the system has sufficient sensing and actuation capabilities to gather desired information. We show that perceivability relates to the reachability of an augmented system, and derive the corresponding Hamilton-Jacobi-Bellman equations to determine perceivability. In simulations, we demonstrate how clarity is a useful concept for planning trajectories, how perceivability can be determined using reachability analysis, and how a Control Barrier Function (CBF) based controller can dramatically reduce the computational burden. 
\end{abstract}


\section{Introduction}
\label{sec:introduction}
Robots are often deployed to explore unknown or unstructured environments, e.g., ocean gliders perform data collection for remote sensing, or aerial robots search for targets in a disaster response. In this paper, we establish two concepts: \emph{clarity} and \emph{perceivability}, to capture information acquisition and show how it can be used to design informative controllers. 



Information theory has long been used in robotic path planning. Informative Path Planning (IPP) seeks to design trajectories that maximize the `amount of information' collected subject to budgetary constraints such as total energy or total time~\cite{hollinger2013sampling}. `Information' is measured in many ways, including entropy/mutual information~\cite{cao2013multi, moon2022tigris}, Fisher Information~\cite{zhang2019beyond}, through the number of unexplored cells/frontiers~\cite{cai2009information, zhou2021fuel}, the area of Voronoi partitions~\cite{cortes2004coverage, jiang2015higher}, or Gaussian Processes~\cite{popovic2020informative, marchant2014bayesian}.  Various techniques to solve IPP have been proposed, including grid/graph-search techniques and sampling-based techniques~\cite{bry2011rapidly, moon2022tigris, cao2013multi, xiao2022nonmyopic}. While useful for trajectory generation, such methods cannot quantify whether information can be gathered in the first place.

The main objective in this letter is to answer the following two questions: Given a platform (e.g., a robot) with onboard sensors, and an environment in which information is to be collected, (1) Does the overall system have sufficient actuation and sensing capabilities to gather information in a specified time, and (2) what are optimal control strategies to collect the information?

To address these, we introduce the notion of \emph{clarity} as a measure of the quality of information possessed about a variable $m$. Clarity about a random variable $m$, denoted $q$, lies in $[0, 1]$, where $q=0$ corresponds to the case where $m$ is completely unknown, and $q=1$ to the case where $m$ is perfectly known in an idealized (noise-free) setting. 

As a first contribution, we show that if $m$ is estimated using a Kalman Filter, the rate of change of clarity has a similar structure and response to one assumed in dynamic coverage controllers~\cite{haydon2021dynamic, panagou2016distributed, bentz2019hybrid}. This establishes certain optimality properties for dynamic coverage control, rather than being viewed as a heuristic approach to exploration. From an information theoretic perspective, clarity of $m$ is injective wrt differential entropy $m$, but is bounded between $[0, 1]$ instead of $[-\infty, \infty]$, with dynamics that are well defined at both ends $q=0, 1$. This is computationally easier to handle, akin to the difference between reciprocal and zeroing CBFs~\cite{ames2016control}.

With the notion of clarity at hand, we then define \emph{perceivability}, which aims to capture the maximum clarity that can be achieved in a fixed time by a given sensing system (robot dynamics and sensory outputs). We show that perceivability can be determined using reachability analysis, i.e., using a Hamilton-Jacobi-Bellman (HJB) equation. This allows us to compute optimal controllers that maximize clarity. Simulation studies demonstrate the concepts of clarity and perceivability, and we demonstrate how CBFs can alleviate the computation burden in HJB based methods.



\subsubsection*{Notation} Let $\R, \Rnonneg, \Rplus$ be the set of reals, non-negative reals, and positive reals respectively. $\overline{\R} = \R \cup \{\pm \infty\}$. Let $\pd^n$, $\psd^n$ denote the set of symmetric positive-definite and symmetric positive-semidefinite matrices in $\R^{n\times n}$. The determinant and trace of a square matrix $P$ are denoted $\det{P}, \tr{P}$ respectively. Let $\Ucal(a, b)$ denote a uniform distribution on the interval $[a, b] \subset \R^n$. Let $\Ncal(\mu, \Sigma)$ denote a normal distribution with mean $\mu \in \R^n$ and covariance $\Sigma \in \pd^n$.

\section{Clarity}
\label{sec:clarity}

To aid the reader, we use a running example throughout the paper, inspired by an oceanographic survey mission: we wish to create a map of the ocean-surface temperature in a specified region. The temperature measurements are obtained by sensors onboard a surface vessel, or from thermal images on an aerial vehicle, both subject to ocean currents or winds. 

Since we need a suitable information metric, we propose clarity, defined next. 

\subsection{Definitions and Fundamental Properties}

\begin{definition}\cite[Ch. 8]{thomas2006elements}
$X$ is a \emph{continuous random variable} if its \emph{cumulative distribution} $F(x) = Pr(X \leq x)$ is continuous. The \emph{probability density function} is $f(x) = F'(x)$. The set where $f(x) > 0$ is the \emph{support set} of $X$. 
\end{definition}

\emph{Differential entropy} extends the notion of entropy for discrete random variables (defined by Shannon~\cite{shannon1948mathematical}) to continuous random variables:

\begin{definition}\cite[Ch. 8]{thomas2006elements}
The \emph{differential entropy} $h[X]$ of a continuous random variable $X$ with density $f(x)$ is 
\eqn{
h[X] = - \int_{S} f(x) \log f(x) dx \label{eqn:diff_entropy}
}
where $S$ is the support set of $X$.
\end{definition}

While differential entropy shares many of the same properties as discrete entropy~\cite[Sec. 2.1]{thomas2006elements}, there are some key differences. For example, while discrete entropy lies in $[0, \infty]$, differential entropy lies in $[-\infty, \infty]$, i.e., entropy can be negative.  We define \emph{clarity} as:
\begin{definition}
Let $X$ be a $n$-dimensional continuous random variable with differential entropy $h[X]$. The \emph{clarity} of $X$ is
\eqn{
q[X] = \left(1+\frac{\exp{(2 h[X])}}{(2 \pi e)^n}\right)^{-1}. \label{eqn:clarity}
}
\end{definition}
The normalizing factor $(2\pi e)^n$ is introduced to simplify some of the algebra, as demonstrated in the example:

\begin{example}
\label{ex:differential_entropy_and_clarity}
Consider $X \sim \Ucal(a, b)$, and $Y \sim \Ncal(\mu, P)$, where $a, b \in \R$, $\mu \in \R^n$, $P \in \psd^n$. The differential entropy and clarity of $X, Y$ are
\eqnN{
&h[X] = \log{(b-a)}, \quad
h[Y] = \log {\sqrt{(2 \pi e)^n \det{P}}}\\
&q[X] = \frac{1 }{1 + \frac{(b-a)^2}{2  \pi e} }, \quad
q[Y] = \frac{1}{1 + \det{P}}.
}
\end{example}

Next, we establish some fundamental properties of clarity. 

\begin{property}
For any $n$-dimensional continuous r.v. $X$, $A \in \R^{n \times n}$, and  $c \in \R^n$, 
\eqn{
&q[X]  \in [0, 1] && \text{(clarity is bounded)} \label{eqn:bounded}\\
&q[X + c] = q[X] && \text{(clarity is shift-invariant)} \label{eqn:shift_inv}\\
&q[AX] \neq q[X] && \text{(clarity is \emph{not} scale-invariant)} \label{eqn:scale_inv}
}
\end{property}
\begin{proof}
\emph{Of \eqref{eqn:bounded}:} Since $h[X] \in \overline{\R}$,  $q[X] = 1 / (1 + s)$ for some $s \in [0, \infty]$, i.e., $q[X] \in [0, 1]$. 

\emph{Of \eqref{eqn:shift_inv}, \eqref{eqn:scale_inv}:} Follows from \cite[Th. 8.6.3]{thomas2006elements} ($h[X + c] = h[X]$), and \cite[Th. 8.6.4]{thomas2006elements} ($h[A X] = h[X] + \log{\abs{a}}$).
\end{proof}

In information gathering tasks, we seek to design trajectories that minimize the estimation error: Let $X$ be a random variable of \emph{any} distribution with clarity $q[X]$. Let $\hat X$ be \emph{any} estimate of $X$, and $E[(X-\hat X)(X-\hat X)^T]$ be defined as the expected estimation error. In Theorem~\ref{thm:pred_error_mv}, and Cor.~\ref{thm:pred_error} we show that clarity bounds the expected estimation error:  a necessary condition for expected estimation error to approach 0 is that clarity must approach 1.

\begin{theorem}
\label{thm:pred_error_mv}
For any $n$-dimensional continuous random variable $X$ and any $\hat X \in \R^n$, 
    the determinant of the expected estimation error is lower-bounded as
    \eqn{
    \det{E[(X-\hat X)(X-\hat X)^T]} \geq \frac{1}{q[X]} - 1,
    }
    with equality if and only if $X$ is Gaussian and $\hat X = E[X]$. 
\end{theorem}
\begin{proof}
Following the same arguments as in~\cite[Th. 8.6.6]{thomas2006elements},
\eqnN{
&\det{E[(X-\hat X)(X-\hat X)^T]}  \geq \min_{\hat X \in \R^n} \det{ E[(X - \hat X) ( X - \hat X)^T]}\\
& \quad = \det{ E[(X - E[X]) ( X - E[X])^T]} = \det{\operatorname{var}(X)}
}
and since a Guassian distribution has the greatest entropy of a given variance~\cite[Th. 8.6.6]{thomas2006elements}, 
\eqnN{
\det{E[(X-\hat X)(X-\hat X)^T]} \geq \frac{e^{2 h[X]}}{(2 \pi e)^n} = \frac{1}{q[X]} - 1.
}\end{proof}
\begin{corollary}
\label{thm:pred_error}
For any $1$-dimensional continuous random variable $x$ and any $\hat x \in \R$, the expected estimation error is lower bounded by 
    \eqn{
    E[(x-\hat x)^2] \geq \frac{1}{q[x]} - 1
    }
    with equality if and only if $x$ is Gaussian and $\hat x = E[x]$.
\end{corollary}
\begin{proof}
    Use Thm.~\ref{thm:pred_error_mv} with $P \in \pd^1 \implies \det{P}=P$.
\end{proof}

\subsection{Clarity and Coverage Control}
Next, we demonstrate the connection between clarity and coverage control. Consider the system 
\eqn{
\dot x = f(x, u) \label{eqn:xdot}
}
where $x \in \Xcal \subset \R^n$ is the system sate, $u \in \Ucal \subset \R^m$ is the control input, and $f: \Xcal \times \Ucal \to \R^m$ defines the dynamics. 

The problem in coverage control is to design a controller $\pi : \Xcal \to \Ucal$ for the system~\eqref{eqn:xdot} such that closed-loop trajectories gather information over a domain $\Dcal \subset \Xcal$. As in~\cite{haydon2021dynamic}, let $c = c(t, p)$ denote the `coverage level' about a point $p \in \Dcal$ at time $t$. \cite{haydon2021dynamic} assumes the coverage level increases through a sensing function $S: \Xcal \times \Dcal \to \R$ (positive when $p$ can be sensed from $x$, and 0 else), and coverage decreases at a rate $\alpha: \Dcal \to \Rnonneg$. This results in the model
\eqn{
\dot c = S(x, p) (1 - c) - \alpha(p) c. \label{eqn:coverage}
}
In~\cite{panagou2016distributed, bentz2019hybrid} the $\alpha$ term is ignored, and a point $p$ is said to be `covered' if $c(t, p)$ reaches a threshold $c^*$. 

However, given specifications on the robot, sensors, and the environment, it is not clear how to systematically define $S, \alpha, c^*$. \cite{bentz2019hybrid, haydon2021dynamic} resort to heuristic methods. 

In many practical scenarios, measurements are assimilated using a Kalman Filter. In principle, the coverage dynamics should reflect the information gathering mechanism, i.e., the quality of information as the environment is estimated using the Kalman Filter. In deriving the clarity dynamics, final result in~\eqref{eqn:clarity_dynamics}, we will notice a similarities with~\eqref{eqn:coverage}.


Consider the simplest scenario, where we want to estimate a scalar variable $m$. We assume $m$ is a stochastic process:
\eqn{
\dot m &= w(t), && w(t) \sim \Ncal(0, Q), \label{eqn:mdot}\\
y  &= C(x) m + v(t), && v(t) \sim \Ncal(0, R(x)), \label{eqn:y}
}
where $m \in \R$ is the quantity of interest. Given the robot is at a state $x$, the (scalar) measurement obtained is $y \in \R$, and can be perturbed by some measurement noise $v(t)\sim \Ncal(0, R(x))$. Notice that $C(x), R(x)$ are state dependent, emphasizing that the quality of the measurements of $m$ can depend on the robot's state $x$. For simplicity, assume the state $x$ is known exactly. The following demonstrates the setup:
\begin{example}
Let $x$ be the quadrotor's state, with position $x_{pos} \in \R^2$ and altitude $x_{alt}$. The quadrotor uses a downward facing thermal camera with half-cone angle $\theta$ to measure the ocean's  temperature $m$ at a location $p$. Then $C(x)$ is 
\eqnN{
C(x) = \begin{cases}
    1, & \text{if }\norm{x_{pos} - p} \leq x_{alt} \tan{\theta},\\
    0, & \text{else }
\end{cases}
}
and (if the measurement variance is state-independent),  $R(x) = R_0$ for some known $R_0$. The fact that the ocean temperate can change stochastically is reflected by~\eqref{eqn:mdot}.
\end{example}

Notice that the subsystem \eqref{eqn:mdot}, \eqref{eqn:y} satisfies the assumptions of linear-time varying Kalman Filters~\cite[Ch. 4]{gelb1974applied}, since for any given trajectory $x(t)$, the measurement model is equivalent to $y = C(t) m + v(t)$, where $C(t) = C(x(t))$ by slight abuse of notation. Therefore, the estimate has distribution $\Ncal(\mu, P)$, where $\mu, P$ evolve according to:
\eqnN{
\dot {\mu} = P C(x) R(x)^{-1} ( y - C(x) \mu), \quad \dot P = Q - \frac{C(x)^2}{R(x)} P^2.
}
where $P \in \Rplus$ is the variance of the estimate. Since the clarity of a scalar Gaussian distribution is $q = 1/(1+P)$, 
\eqnN{
\dot q &= \frac{\partial q}{\partial P} \dot P = \frac{-\dot P}{(1+P)^2} = \frac{-1}{(1+P)^2} \left(Q - \frac{C(x)^2}{R(x)} P^2\right)\nonumber
}
and therefore the clarity dynamics are
\eqn{
\dot q = \frac{C(x)^2}{R(x)} (1 - q)^2 - Q q^2. \label{eqn:clarity_dynamics}
}

\begin{remark}
    Comparing~\eqref{eqn:coverage} with~\eqref{eqn:clarity_dynamics}, one may note that their structure is remarkably similar. Clarity and coverage increase due to the first term, and decrease due to the second. However, \eqref{eqn:clarity_dynamics} is nonlinear wrt $q$. Thus, although~\eqref{eqn:coverage} has the right intuitive characteristics to describe `coverage', \emph{\eqref{eqn:clarity_dynamics} has the correct dynamics corresponding to information gathering}, i.e., the rate of improvement of the estimate.  
\end{remark}

Equation~\eqref{eqn:clarity_dynamics} yields further insight. Clarity decays at a rate $-Q q^2$, i.e., related to the stochasticity of the environment. Furthermore, the incremental value of a measurement decreases as the clarity increases: $C(x)^2 (1-q)^2/R(x)$ decreases as $q$ increases. In other words, although every additional measurement increases clarity, there are diminishing returns, quantified by~\eqref{eqn:clarity_dynamics}.

Although nonlinear, \eqref{eqn:clarity_dynamics} has closed-form solutions, since it is an instance of a (scalar) differential Riccati equation~\cite[Sec. 2.15]{ince1927ordinary}. For constant $C(x) = C, R(x) = R$, if $C, R, Q > 0$, the solution is 
\eqn{
    q(t) = q_\infty \left( 1 + \frac{2\gamma_1}{\gamma_2 + \gamma_3 e^{2 k Q t}} \right), \label{eqn:closed_form_clarity}
    }
where $k = C/\sqrt{Q R}$, $q_\infty = k/(k+1)$, $\gamma_1 = q_\infty - q_0$, $\gamma_2 = \gamma_1 (k - 1)$, $\gamma_3 = (k-1) q_0 - k$.
    
As $t \to \infty$,  clarity monotonically approaches $q_\infty < 1$ for $Q, R \neq 0$. Thus if $m$ is a stochastic process with non-zero variance, and the measurements have non-zero variance, perfect clarity ($q=1$) cannot be attained. 

The vector case also has the same structure:

\begin{theorem}
    Let $m \in \R^{n_m}$ be the environment state vector, and $y \in \R^q$ be the sensed outputs. Suppose the environment and measurement models are
    \eqn{
    \dot m &= A m + w(t) && w(t) \sim \Ncal(0, Q)\\
    y &= C(x) m + v(t) && v(t) \sim \Ncal(0, R(x))
    }
    with $Q \in \pd^{n_m}$, and $R : \Xcal \to \pd^{q}$. Assuming $P(t) \in \pd^{n_m}$ for all $t$,\footnote{This is a standard assumption in Kalman filtering, amounting to an assumption on the observability of $m$. See~\cite[Sec. 11.2]{khalil2015nonlinear} for details.} and a prior $m \sim \Ncal(\mu, P)$, then 
    \eqn{
    \dot P &= AP + PA^T+ Q - PC(x)^T R(x)^{-1} C(x) P\label{eqn:dotP}\\
    \dot q &= q(1-q) (\trace{(C(x)^T R^{-1} C(x) P)} - \trace{(2A + P^{-1} Q)} ). \label{eqn:dotq_mv}
    }
    
\end{theorem}
\begin{proof}
Eq.~\eqref{eqn:dotP} is the standard covariance update for the Kalman Filter. To derive~\eqref{eqn:dotq_mv}, notice the clarity of a multivariate Gaussian is $q= 1/(1 + \det{P}) $. Therefore, 
\eqnN{
\dot q = -\frac{1}{(1 + \det{P})^2} \frac{d}{dt} (\det{P})
}
Since $P \in \pd^{n_m}$, it is is invertible. Using Jacobi's formula:
\eqnN{
\dot q &= \frac{-\det{P} \trace{( P^{-1} \dot P)} }{(1 + \det{P})^2} = q(1-q) \trace{(-P^{-1} \dot P)}
}
since $\det{P}/(1 + \det{P})^2 = q(1-q)$. Substituting in~\eqref{eqn:dotP}, and simplifying, we arrive at~\eqref{eqn:dotq_mv}.
\end{proof}

Again, we see the same structure: when $C(x) \neq 0$, the clarity increases at a rate proportional to $\trace{(C(x)^T R(x)^{-1} C(x) P)}$, and decreases at a rate proportional to $\trace{(P^{-1} Q)}$. Furthermore, since clarity dynamics are independent of $y$,  for trajectory planning purposes we can consider the deterministic (and fully known) system
\eqn{
\dot X = \tilde f(X, u), \quad 
\dot q = g(X, q)
}
where $X = [x^T, \operatorname{vec}(P)^T]^T$ is an extended state.

\section{Perceivability}

In this section, we introduce the concept of \emph{perceivability}, which measures the following: given a robot with certain sensing and actuation capabilities, can the robot's motion over a finite time achieve a desired level of clarity with the collected sensory data? 
Formally,
\begin{definition}
    A quantity $m \in \R$ that evolves according to~\eqref{eqn:mdot} is \emph{perceivable} by the system~(\ref{eqn:xdot}, \ref{eqn:y}) with clarity dynamics\footnote{When using a Kalman Filter to estimate $m$, $g$ is as in~\eqref{eqn:clarity_dynamics}. In general, other estimators could be used, and will lead to different expressions for $g$. } $g: \Xcal \times [0, 1] \to \R$, to a level $q^* \in [0, 1]$ at time $T$ from an initial state $x_0 \in \Xcal$ and clarity $q_0 \in [0,1]$, if there exists a controller $\pi: [0, T] \to \Ucal$ s.t. the solution to 
    \eqn{
    \bmat{\dot x\\\dot q} = \bmat{f(x, \pi(t)) \\ g(x, q)}, \quad \bmat{x(0)\\q(0)} = \bmat{x_0\\ q_0}\label{eqn:perc}
    }
    satisfies $q(T) \geq q^*$. 
\end{definition}

We define the set of initial conditions from which $m$ is perceivable as the \emph{perceivability domain}:

\begin{definition}
    The $(q^*, T)$-\emph{Perceivability Domain} of a quantity $m \in \R$ (that evolves according to~\eqref{eqn:mdot}) by the system (\ref{eqn:xdot}, \ref{eqn:y})  is the set of initial states $x_0$ and initial clarities $q_0$ such that $m$ is perceivable to a level $q^*$ at time $T$:
    \eqn{
    \Dcal(q^*, T) = \Big \{& (x_0, q_0) : \exists \pi: [0, T] \to \Ucal, \notag\\
    & \quad \dot x = f(x, \pi(t)), \;  \dot q = g(x, q), \notag \\
    & \quad x(0) = x_0, \; q(0) = q_0, \; q(T) \geq q^* \Big\}.
    }
\end{definition}

Our key insight is that perceivability is fundamentally a question of the reachability of the augmented system~\eqref{eqn:perc}.  As with backward reachable sets, the perceivability domain can be defined by a Hamilton-Jacobi (HJB) equation:
\begin{theorem}
Let $V: [0, T] \times \Xcal \times [0, 1] \to \R$ solve
    \eqn{
    &\frac{\partial V}{\partial t} + \max_{u \in \Ucal} \left(\frac{\partial V}{\partial x}f(x, u) \right) + \frac{\partial V}{\partial q} g(x, q) = 0, \label{eqn:HJB}\\
    &V(T, x, q) = q \quad \forall x \in \Xcal, q \in [0, 1]. \label{eqn:HJB_bc}
    }
    Then the $(q^*, T)$- perceivability domain of $m \in \R$ (that evolves according to~\eqref{eqn:mdot}) by the system (\ref{eqn:xdot}, \ref{eqn:y})  is
    \eqn{
    \Dcal(q^*, T) = \left\{ [x_0^T, q]^T : V(0, x_0, q_0) \geq q^* \right\}.
    }
\end{theorem}
\begin{proof}
Let $\Lcal([t, T], \Ucal)$ be the set of piecewise continuous functions $\pi:[t, T] \to \Ucal$. Define $V$ as the maximum clarity reachable from $(t, x, q)$:
\eqnN{
V(t, x(t), q(t)) = &\max_{\pi \in \Lcal([t, T], \Ucal)}  q(T) \; \text{s.t. } \eqref{eqn:perc}
}
By the principle of dynamic programming, for any $\delta > 0$,
\eqnN{
V(t, x(t), q(t)) = &\max_{\pi \in \Lcal([t, t+\delta], \Ucal)} V(t+\delta, x(t + \delta), q(t+\delta))
}
Using a Taylor expansion about $\delta = 0$, as $\delta \to 0$, 
\eqnN{
V(t, x(t), q(t)) &= \max_{u \in \Ucal} \Big( V(t, x(t), q(t)) + \frac{\partial V}{\partial t} \delta \notag \\ 
&\quad + \frac{\partial V}{\partial x} f(x, u) \delta  + \frac{\partial V}{\partial q} g(x, q)\delta \Big)
}
which simplifies to \eqref{eqn:HJB}.
\end{proof}

As with standard reachability theory, once the HJB equation~(\ref{eqn:HJB}, \ref{eqn:HJB_bc}) has been solved, the optimal controller is
\eqn{
\pi(t, x, q) = \argmax{u \in \Ucal} \left(\frac{\partial V}{\partial x} f(x, u) \right) \label{eqn:hjb_u}
}



\section{Simulations and Applications}
\subsection{Energy-Aware Information Gathering}
\footnote{All code and videos are available at~\cite{clarityRepo}.} This example demonstrates that the incremental value from measurements decreases as clarity increases. 
Consider the quadrotor tasked with measuring the ocean temperature. It must fly to a target location, spend $T$ seconds collecting information, and fly back to the start. As $T$ increases, more measurements are made and hence greater clarity is achieved, but at the cost of additional energy use. We wish to determine optimal $T$ to maximize clarity and minimize energy. We model the energy cost as $E(t) = p_0 + p_1 T$, where $p_0$ is the energy cost of flying to and back from the target, $p_1$ is the power draw at hover. The clarity dynamics are as in~\eqref{eqn:clarity_dynamics}.

The pareto front of $q(T)$ against $E(T)$ is depicted in Fig.~\ref{fig:clarity_energy}.  The diminishing value of measurements is clearly visible, as between $T \in [160, 320]$~s, the clarity only increases by 2.6\%, but increases by 49.7\% for $t \in [10, 20]$~s. To maximize the clarity/energy ratio, the quadrotor should collect measurements for $T^* =57.4$~seconds (green tangent).

\begin{figure}
    \centering
    \includegraphics[width=0.9\linewidth]{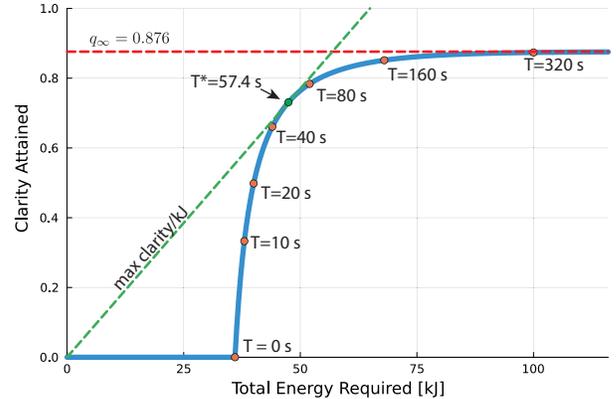}
    \caption{Clarity gained as a function of  the measurement time. First, the clarity increases rapidly (between $T=30$ to 40~s). As the level of clarity approaches $q_\infty$ (red dashed line), the rate of clarity accumulation decreases. The maximum clarity/energy ratio is (green dashed line) is achieved at $T^* = 57.4$~s. Parameters: $R = 20.0$, $Q = 0.001$, $p_0 = 36$~kJ, $p_1 = 0.2$~kW. }
    \label{fig:clarity_energy}
    \vspace{-3mm}
\end{figure}



\subsection{Coverage Control based on Clarity}

\begin{figure*}[t]
    \centering
    \includegraphics[width=\linewidth]{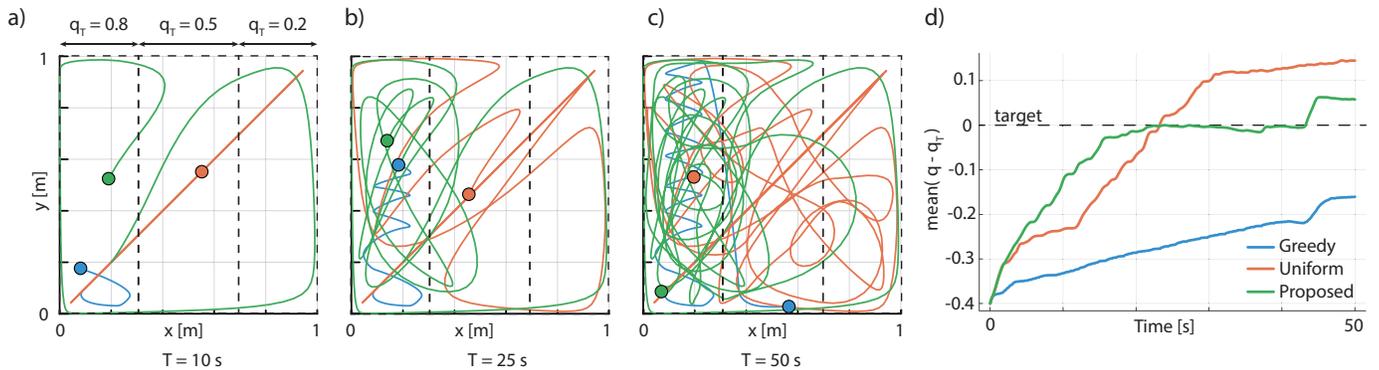}
    \caption{Coverage Controllers. (a-c) Snapshots of three controllers exploring a square region. The target clarity $q_T(p)$ is different in different regions as labelled in (a). (d) Plot of the mean($q(t, p) - q_T(p)$) against $t$ for each controller. Notice that using the proposed method, the mean clarity error is close to 0 for $t \in [20-35]$ seconds, and only increases later, when the entire region has higher clarity than the targets specified. }
    \label{fig:ergodic}
    \vspace{-3mm}
\end{figure*}

Next, we demonstrate how clarity can be used in ergodic coverage controller of~\cite{mathew2011metrics}. The robot is exploring a unit square, but certain regions have a greater target clarity than others, as labelled in Fig.~\ref{fig:ergodic}a. The challenge with ergodic controllers is defining the fraction of time spent at each position $p$, and uniform allocation is used as a heuristic. Since the target clarity has been specified, we can invert \eqref{eqn:closed_form_clarity} to determine the appropriate time allocation.

In Fig.~\ref{fig:ergodic}, we compare the behaviour of three coverage controllers: (A) a greedy controller drives to a point with maximum $(q_T(p) - q(t, p))$ and hovers at $p$ until $q_T$ is reached, (B) the ergodic controller in \cite{mathew2011metrics} with a uniform target distribution, and (C) the same ergodic controller but with a target distribution based on clarity. The proposed method brings the mean of ($q(t, p) - q_T(p)$) to 0 rapidly, and does not overshoot like controller B. Beyond $t=35$, most cells have reached the target clarity, and since the robot continues to explore, $q(t, p)$ increases further.

\subsection{Perceivability and Optimal Trajectory Generation}

\begin{figure*}[t]
    \centering
    \includegraphics[width=\linewidth]{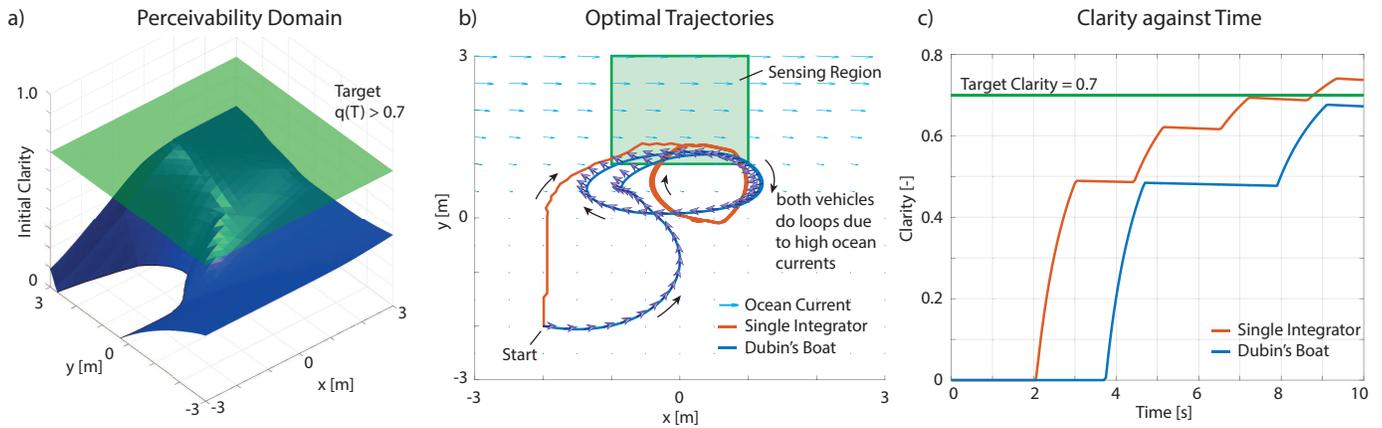}
    \caption{Perceivability and Optimal Trajectories. (a) The $(q^*, T)$-Perceivability Domain (states above blue surface) for single integrator using $q^* = 0.7, T=10.0$~sec. (b) Optimal trajectories for the single integrator (orange) and the Dubins boat (blue) from the same initial conditions. The heading of Dubins boat is shown with blue arrows. Due to the high ocean currents in the sensing region, both vehicles make multiple passes through the sensing region to accumulate clarity. (c) Plot of clarity against time for both vehicles. Since the single integrator is more maneuverable than the Dubins boat, the environment is perceivable to a level 0.7 in time 10~seconds for the single integrator but not for the Dubins boat.}
    \vspace{-3mm}
    \label{fig:hjb}
\end{figure*}

Here we demonstrate how the perceivability can be determined using~\eqref{eqn:HJB}, \eqref{eqn:HJB_bc}. Consider a boat tasked with collecting information that can only be measured from a specified region (green square in Fig.~\ref{fig:hjb}b). To highlight the importance of actuation capabilities on perceivability, we consider two models, a single integrator:
\eqnN{
\dot x_1 = u_1 + w_x(x), \; \dot x_2 = u_2 + w_y(x)
}
with $u_1, u_2 \in [-2, 2]$~m/s, and a Dubins Boat:
\eqnN{
\dot x_1 = v \cos x_3 + w_x(x),\; \dot x_2 = v \sin x_3 + w_y(x), \; \dot x_3 = u
}
where $v = 2$~m/s, and $u \in [-1,1]$~rad/s. For both the ocean current is $w_x(x) = \max(0, 3x_2), w_y(x) = -0.5$~m/s. Thus, neither vehicle has sufficient control authority to remain within the sensing range. For both vehicles the sensing model is as in~\eqref{eqn:clarity_dynamics}, with $C(x) = 1$ when $x$ is in the green square and $C(x) = 0$ elsewhere, $R(x) = 1.0$, $Q = 0.001$.  

To determine the perceivability domain, the backwards reachability set of~\eqref{eqn:perc} is computed using~\cite{mitchell2005toolbox, bansal2017hamilton}. Fig.~\ref{fig:hjb}a shows the perceivability domain for the single integrator. The optimal controller~\eqref{eqn:hjb_u} is used to drive both vehicles from the same initial condition, and the resulting trajectories are plotted in~Fig.~\ref{fig:hjb}b. Due to the ocean currents, both vehicles need to do loops to acquire clarity. Clarity is plotted against time in Fig.~\ref{fig:hjb}c, and we see that the single integrator is able to reach $q(T) \geq q^*$, but Dubins boat is not. Despite having the same sensing capabilities, the perceivability is different due to different actuation capabilities. 

Computing the 10-second perceivability domain took 450~seconds on a Macbpoook Pro (i9, 2.3GHz, 16GB). While prohibitively slow for online applications, $V$ can be precomputed offline. Future work will explore fast trajectory design techniques, and consider safety or energy constraints using tools akin to RIG~\cite{hollinger2013sampling}, or CBFs, as demonstrated next.

\subsection{CBF-based Trajectory Generation}
If a CBF~\cite{ames2016control} can be found for a system, one does not need to solve~\eqref{eqn:HJB}. Consider a 6D planar quadrotor~\cite{agrawal2022safe}:
\eqnN{
\ddot x_1 = u_1 \sin{x_3}/ m, \; \ddot x_2 = u_1 \cos{x_3}/ m - g,  \; \ddot x_3 = u_2 / J
}
where $x_1, x_2$ is the position of the quadrotor in the vertical plane, and $x_3$ is the pitch angle. $m, g, J$ are the mass, acceleration due to gravity, and moment of inertia of the quadrotor. The quadrotor is attempting a precision landing, using onboard sensors to determine the landing spot. To prevent the quad from descending too quickly, we impose the constraint $x_2 \geq 2 \sigma$, where $\sigma$ is the std of the estimated landing site. Using $\sigma^2 = 1/q - 1$, this reads
\eqnN{
\Scal =  \{ [x^T, q]^T : h(x, q) = q - 4/(4 + x_2^2)  \geq 0\}
}
where $h$ is a CBF of relative degree 2 for the planar quadrotor system. Fig.~\ref{fig:landing} shows the trajectories with and without the CBF-QP controller~\cite{xiao2019control}. The CBF controller slows down to collect sufficient quality of information before landing. We attempted to solve the same problem using~\eqref{eqn:HJB}, but the since the system is 7D (6D for $x$, 1D for clarity), it took 480~s to compute the 0.05~s perceivability domain on a coarse grid. Using a slightly finer grid required over 60GB of RAM, and MATLAB crashed. In contrast, the CBF-QP controller computes safe control inputs in under a 1~ms. 

\begin{figure}
    \centering
    \includegraphics[width=0.8\linewidth]{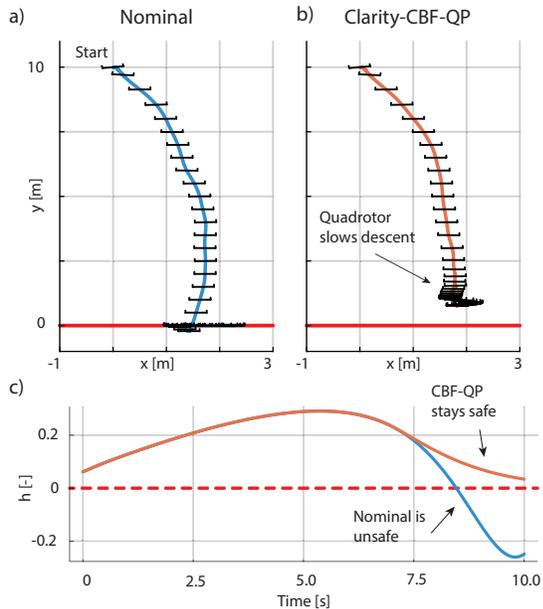}
    \caption{Precision landing of a planar quadrotor. (a) In the nominal controller, the quad descends rapidly and misses the target. (b) Using the clarity based CBF-QP controller, the quad descends slowly. (c) Plot of $h$ against $t$, showing the CBF-QP keeps the system safe.}
    \label{fig:landing}
    \vspace{-3mm}
\end{figure}

\section{Conclusion}

In this paper we have introduced the concepts of clarity and perceivability. While clarity is simply a redefinition of entropy, we show an interesting connection between coverage control and information theory through clarity. Furthermore the algebraic simplicity of expressions involving clarity make it intuitive for control design. We remark that although clarity dynamics of the Kalman Filter  are nonlinear, closed form solutions exist. We defined perceivability of an environment as the ability for a sensing and control system to collect information, measured in terms of the clarity that can be gained. This allows us to interpret and quantify the information gathering capabilities of a system in terms of reachability analysis, a well-established field with a large set of mathematical and software tools. In the future, we hope to develop computationally-efficient tools to analyze perceivability. 

\appendices


{
\setstretch{0.925}
\bibliographystyle{ieeetr}
\bibliography{biblio}
}

\clearpage

\end{document}